\documentclass{article}
\usepackage{graphicx} 

\usepackage{multirow}
\usepackage{tabularx}
\usepackage[utf8]{inputenc}
\usepackage{fullpage}
\usepackage{amsmath}
\usepackage{amsthm}
\theoremstyle{plain}
\usepackage{amssymb}
\usepackage[ruled,vlined]{algorithm2e}
\usepackage{xcolor}
\usepackage{hyperref}
\usepackage{mathtools}
\usepackage{mdframed}

\newtheorem{claim}{Claim}
\newtheorem{lemma}{Lemma}
\newtheorem{theorem}{Theorem}
\newtheorem{definition}{Definition}

\newtheorem{fact}{Fact}
\newtheorem{corollary}{Corollary}
\newtheorem{observation}{Observation}

\newcommand{\remove}[1]{}

\newcommand{\calX}{{\cal X}}
\newcommand{\calY}{{\cal Y}}

\title{Private Prediction via Shrinkage}
\author{Chao Yan\thanks{Department of Computer Science, Georgetown University. {\tt cy399@georgetown.edu}. Research supported in part by a gift to Georgetown University.}}

\date{\today}

\begin{document}

\maketitle

\begin{abstract}
We study \emph{differentially private prediction} introduced by Dwork and Feldman~\cite{DworkF18} (COLT 2018): an algorithm receives one labeled sample set $S$ and then answers a stream of unlabeled queries while the output transcript remains $(\varepsilon,\delta)$-differentially private with respect to $S$. Standard composition yields a $\sqrt{T}$ dependence for $T$ queries. 

We show that this dependence can be reduced to \emph{polylogarithmic} in $T$ in streaming settings. For an oblivious online adversary and any concept class $\mathcal{C}$, we give a private predictor that answers $T$ queries with $|S|= \tilde{O}(VC(\mathcal{C})^{3.5}\log^{3.5}T)$ labeled examples. For an adaptive online adversary and halfspaces over $\mathbb{R}^d$, we obtain $|S|=\tilde{O}\left(d^{5.5}\log T\right)$.
\end{abstract}

\section{Introduction}
In the PAC learning model~\cite{Valiant84}, a learner receives a labeled sample $S=\{(x_i,y_i)\}_{i=1}^N$, where each $x_i$ is drawn i.i.d. from an unknown distribution $\mathcal{D}$ over a domain $X$, and each label is generated by an unknown target concept $c\in\mathcal{C}$, where $\mathcal{C}\subseteq\{c:X\rightarrow\{\pm 1\}\}$. The learner outputs a hypothesis $h$ intended to predict the labels of fresh unlabeled points. It is well known that the sample complexity of PAC learning is characterized (up to constant factors) by the VC dimension $VC(\mathcal{C})$.

In many applications, however, the sample $S$ contains sensitive information, and the released hypothesis $h$ may leak details about individual examples. Kasiviswanathan et al.~\cite{KLNRS08} introduced private learning, which requires the learner to satisfy differential privacy with respect to the input sample $S$. They gave a generic differentially private learner with sample complexity $O(\log|\mathcal{C}|)$. Since $VC(\mathcal{C})\leq\log|\mathcal{C}|$, this implies privacy can incur a potentially large overhead: the gap between $VC(\mathcal{C})$ and $\log|\mathcal{C}|$ can be arbitrarily large.

By now, it is well understood that private learning can require strictly more samples than non-private PAC learning. One example is the threshold class over an ordered domain $X$: $\mathcal{C}=\{c_t:c_t(x)=1\mbox{ if }x\geq t,\;t\in[X]\}$, which has $VC(\mathcal{C})=1$, and can be PAC learned from $O(1)$ examples. In contrast, under pure differential privacy, learning thresholds requires $\Theta(\log|X|)$  samples~\cite{FX14}, and under approximate differential privacy it requires $\Theta(\log^*|X|)$ samples~\cite{BNS13b,BNSV15,KaplanLMNS19,AlonLMM19,CohenLNSS22, yan2025optimal}. In particular, for infinite domains, thresholds are PAC learnable but not privately learnable.

Dwork and Feldman~\cite{DworkF18} proposed an alternative paradigm for privacy-preserving prediction. Instead of outputting a hypothesis $h$, the algorithm interacts with a stream of unlabeled queries $x$ and outputs a label only when a query arrives. For a single query, they showed that $|S|=\Theta(VC(\mathcal{C}))$ labeled examples suffice. For answering $T$ queries, the composition theorem yields $|S|=O(VC(\mathcal{C})\cdot\sqrt{T})$ ~\cite{DRV10}. (also observed by Bassily et al.~\cite{BassilyTT18})

A central question is whether one can improve the dependence on $T$. For an \textbf{offline} adversary, who reveals all queries in advance, Beimel et al.~\cite{BNS15} showed that $O(VC(\mathcal{C})\log T)$ labeled examples suffice to privately label $T$ unlabeled points. For the \textbf{stochastic online} setting, where queries are drawn i.i.d. from the same distribution $\mathcal{D}$, Nandi and Bassily~\cite{NandiB20} obtained bounds independent of $T$ via semi-supervised ideas, giving sample complexity
$O((VC(\mathcal{C}))^{1.5})$. Naor et al.~\cite{NNSY23} further showed that one can even protect the privacy of the queries themselves with sample complexity $\tilde{O}((VC(\mathcal{C}))^{2})$ (also see~\cite{Stemmer24} for improved sample complexity on a specific concept class).

These latter results, however, rely crucially on the assumption that online queries are drawn from the underlying distribution $\mathcal{D}$. It remains unclear whether one can beat the $\sqrt{T}$ dependence when the queries are adversarially chosen in an online (streaming) manner. This motivates the following question:
\begin{mdframed}
\begin{center}
        Can we privately answer a super-polynomial number of adversarially chosen queries in the online (streaming) setting?
\end{center}
\end{mdframed}
\subsection{Adversary models.}
We consider several standard models for how the query sequence $x_1,\dots,x_T$ is generated.

\begin{itemize}
    \item \textbf{Offline adversary.}
    The adversary selects the entire query sequence $(x_1,\dots,x_T)$ in advance and reveals it to the predictor before any
    interaction takes place.

    \item \textbf{Oblivious online adversary.}
    The adversary selects the entire sequence $(x_1,\dots,x_T)$ in advance, but reveals the queries one at a time.
    Equivalently, $x_j$ may depend on public parameters (and on $j$), but not on the predictor's internal randomness or prior outputs.

    \item \textbf{Stochastic online adversary.}
    Each query $x_j$ is drawn i.i.d.\ from the same underlying distribution $\mathcal{D}$ that generates the labeled sample $S$
    (and is independent of the predictor's randomness). This is sometimes called the \emph{online i.i.d.} query model.

    \item \textbf{Adaptive online adversary.}
    The adversary chooses queries sequentially, where $x_j$ may depend on the predictor's prior outputs.
    This is the strongest (streaming) adversary model considered in this work.
\end{itemize}

\subsection{Our Result}
In this work, we show two main results:
\begin{enumerate}
    \item \textbf{Oblivious online adversary.} We give a differentially private predictor that can answer $T$ prediction queries using $\tilde{O}((VC(\mathcal{C})^{3.5}\cdot \log^{3.5}T)$ labeled samples.

    \item \textbf{Adaptive online adversary for halfspaces.} For an adaptive online
    adversary and the class of halfspaces over $\mathbb{R}^d$, we give a differentially
    private predictor that answers  $T$ queries using $\tilde{O}(d^{5.5}\cdot \log T)$ labeled samples.
\end{enumerate}

For comparison, we list the previous results in Table~\ref{tab:comparison}
\begin{table}[t]
\centering
\small
\begin{tabular}{llll}
\hline
\textbf{Adversary model} & \textbf{Labeled sample size} & \textbf{Reference} \\
\hline
Online prediction (general), $T$ queries   & $O(\mathrm{VC}(\mathcal{C})\sqrt{T})$ & \cite{DworkF18,BassilyTT18}  \\
Offline adversary (queries revealed in advance)  & $O(\mathrm{VC}(\mathcal{C})\log T)$ & \cite{BNS15} \\
Online stochastic adversary (queries are sampled from $\mathcal{D}$)  & $O(\mathrm{VC}(\mathcal{C})^{1.5})$ (indep.\ of $T$) & \cite{NandiB20} \\
Online stochastic + query privacy  & $\widetilde{O}(\mathrm{VC}(\mathcal{C})^{2})$ (indep.\ of $T$) & \cite{NNSY23} \\
\hline
Oblivious online adversary   & $\widetilde{O}(\mathrm{VC}(\mathcal{C})^{3.5}\log^{3.5} T)$ & this work \\
Adaptive online adversary, halfspaces over $\mathbb{R}^d$  & $\widetilde{O}(d^{5.5}\log T)$ & this work \\
\hline
\end{tabular}
\caption{Comparison of labeled-sample complexity.}
\label{tab:comparison}
\end{table}

\subsection{The main ideas}

\begin{enumerate}
    \item \textbf{Ideas from previous works.}
    We use the framework by Dwork and Feldman~\cite{DworkF18} and Naor et al.~\cite{NNSY23}. More details can be found in Section~\ref{sec:generic framework}.
    \begin{enumerate}
        \item \textbf{Subsample-and-aggregation}
        We follow the standard subsample-and-aggregate framework appearing in a line of works~\cite{DworkF18,BassilyTT18, NandiB20, NNSY23}. We randomly partition the labeled dataset into many disjoint blocks, each of size $O(VC(\mathcal{C}))$. On each block, we run a non-private PAC learner to obtain a hypothesis. Given a query point $x$, we predict using a noisy majority vote over the ensemble of hypotheses.

        \item \textbf{Sparse vector} To answer exponentially many queries while preserving privacy, we use the Sparse Vector technique. Specifically, we use the algorithm \texttt{BetweenThresholds} introduced by Bun et al.~\cite{BunSU16}. Informally, \texttt{BetweenThresholds} allows us to perform a large number of noisy majority votes and only pays privacy when a query is hard (i.e., when approximately half of the hypotheses vote for label 1). A limitation of prior analyses is that the number of hard queries can be large. We solve it by a shrinkage approach.
    \end{enumerate}

    \item \textbf{Shrinking the concept class to bound the number of hard queries.} The key new idea is to ensure that \texttt{BetweenThresholds} stops only $O(\log T)$ times by progressively shrinking the set of candidate hypotheses. Each time \texttt{BetweenThresholds} declares a query hard, we update (shrink) the hypothesis space so that the number of remaining label patterns decreases.
    \begin{enumerate}
        \item \textbf{Oblivious adversary (adversarially selected queries but fixed in advance).} Suppose the query points $x_1,\dots,x_T$ are fixed in advance. Consider the set of possible label sequences
        $
        \mathcal{L} \;=\; \bigl\{(h(x_1),\dots,h(x_T)) : h\in\mathcal{C}\bigr\}.
        $
        By Sauer's lemma, $|\mathcal{L}| \le O(T^{\mathrm{VC}(\mathcal{C})})$. For each hard query $x_t$, we guess a label $\hat{y}_t\in\{\pm 1\}$ uniformly at random and then remove all hypotheses $h$ inconsistent with the guess, i.e., those with $h(x_t)=-\hat{y}_t$. With probability $1/2$, this removes at least half of the remaining label sequences in $\mathcal{L}$. Therefore, after $O(\mathrm{VC}(\mathcal{C})\cdot \log T)$ hard queries, the remaining hypotheses agree on all future queries, and no query will be hard anymore.

        \item \textbf{Halfspaces under an adaptive adversary.} For halfspaces we leverage geometric techniques of Nissim et al.~\cite{NTY25} and Kaplan et al.~\cite{kaplan2020private}. We cast prediction as maintaining the feasibility of a system of linear constraints corresponding to the target halfspace. Each hard query can be interpreted as a hyperplane constraint. By restricting attention to the intersection structure induced by hard queries, we show that after $O(d)$ such events, the remaining feasible region collapses to a single consistent solution, after which all future queries are easy, and no query will be hard anymore.
    \end{enumerate}
\end{enumerate}

\subsection{Discussion and Open Questions}
Our results show that, for fixed privacy and accuracy parameters, one can privately answer super-polynomially many queries even when the queries arrive in an adversarial stream. We highlight several directions for future work.

\begin{enumerate}
    \item \textbf{Improving the dependence on $\mathrm{VC}(\mathcal{C})$ for oblivious adversaries.}
    Compared to the baseline $O(\mathrm{VC}(\mathcal{C})\sqrt{T})$ bound obtained via composition, our guarantees achieve only polylogarithmic dependence on $T$ but incur a larger polynomial dependence on $\mathrm{VC}(\mathcal{C})$.
    Can one obtain a bound of the form $\widetilde{O}\!\left(\mathrm{VC}(\mathcal{C})\cdot \mathrm{poly}(\log T)\right)$?
    Besides being a natural target in its own right, such an improvement would also sharpen bounds in the stochastic setting:
    combined with techniques of semi-supervised learning~\cite{BNS15,NandiB20}, it would potentially reduce the labeled-sample complexity from
    $O(\mathrm{VC}(\mathcal{C})^{1.5})$ to $\widetilde{O}(\mathrm{VC}(\mathcal{C}))$.

    \item \textbf{Adaptive adversaries beyond halfspaces.}
    Our adaptive-adversary result relies on geometric structure specific to halfspaces.
    Can one obtain polylogarithmic dependence on $T$ for \emph{arbitrary} VC classes against adaptive online adversaries?

    \item \textbf{Query privacy.}
    In our model, the predictor's future behavior depends on past queries and outputs, and we only guarantee differential privacy with respect to the labeled sample $S$.
    Is it possible to additionally protect the privacy of the query stream itself (as in~\cite{NNSY23}), while still achieving polylogarithmic dependence on $T$ against adversarial streams?
\end{enumerate}

\section{Prediction Model}
In this section, we formalize the prediction model. We follow the definition by Naor et al.~\cite{NNSY23}\footnote{In the setting of~\cite{NNSY23}, the predictor predicts an infinite number of queries and requires protecting the privacy of queries. We limit it to be $T$ and only protect the privacy of the dataset.}

\begin{mdframed}
\textbf{Private Prediction.}

Let $\mathcal{C}\subseteq\{\pm1\}^X$ be a concept class over a domain $X$.
An unknown target concept $c\in\mathcal{C}$ and an unknown distribution $\mathcal{D}$ over $X$ generate a labeled sample
\[
S=\{(x_i,y_i)\}_{i=1}^N \in (X\times\{\pm 1\})^N,
\qquad x_i \stackrel{i.i.d.}{\sim} D,\ \ y_i=c(x_i).
\]

The interaction then proceeds for $T$ prediction rounds. In round $j\in[T]$:
\begin{enumerate}
    \item The algorithm $\mathcal{A}$ (possibly using internal randomness and the sample $S$) computes a hypothesis $h_j:X\to\{\pm1\}$.
    \item $\mathcal{A}$ receives a query point $x_j\in X$ (chosen by an adversary, possibly adaptively).
    \item $\mathcal{A}$ outputs the predicted label $\mathsf{Label}_j := h_j(x_j)$.
\end{enumerate}
\end{mdframed}

\begin{definition}
    We say that $\mathcal{A}$ is an $(\varepsilon,\delta,\alpha,\beta)$-differentially private predictor for $\mathcal{C}$ over $X$ if it satisfies:
\begin{enumerate}
    \item \textbf{Privacy.} The transcript of outputs $((x_1,\mathsf{Label}_1),\ldots,(x_T,\mathsf{Label}_T))$ is $(\varepsilon,\delta)$-differentially private with respect to the sample $S$.
    \item \textbf{Accuracy.} With probability at least $1-\beta$ over the randomness of $\mathcal{A}$ and the draw of $S$,
    \[
        \mathrm{err}_\mathcal{D}(c,h_j) \le \alpha \quad \text{for all } j\in[T],
    \]
    where $\mathrm{err}_\mathcal{D}(c,h):=\Pr_{x\sim \mathcal{D}}[h(x)\neq c(x)]$.
\end{enumerate}
\end{definition}

\section{Preliminaries}

\subsection{Learning Theory}
\begin{definition}[Vapnik-Chervonenkis dimension~\cite{VC,Haussler1986EpsilonnetsAS}]
    Let $X$ be a set and $R$ be a set of subsets of $X$. Let $S\subseteq X$ be a subset of $X$. Define
    $
    \Pi_R(S)=\{S\cap r \mid r\in R\}.
    $
    If $|\Pi_R(S)|=2^{|S|}$, then we say $S$ is shattered by $R$. The Vapnik-Chervonenkis dimension of $(X,R)$ is the largest integer $d$ such that there exists a subset $S$ of $X$ with size $d$ that is shattered by $R$.
\end{definition}

\begin{definition}[$\alpha$-approximation\footnote{Commonly called $\varepsilon$-approximation. We use $\alpha$-approximation as $\varepsilon$ is used as a parameter of differential privacy.}~\cite{VC,Haussler1986EpsilonnetsAS}]\label{def:alpha-approx}
    Let $X$ be a set and $R$ be a set of subsets of $X$. 
    Let $S \in X^*$ be a finite multi-set of elements in $X$.
    For any $0\leq\alpha\leq 1$ and $S'\subseteq S$, $S'$ is an $\alpha$-approximation of $S$ for $R$ if for all $r\in R$, it holds that
    $
    \left|\frac{|S\cap r|}{|S|}-\frac{|S'\cap r|}{|S'|}\right|\leq \alpha.
    $
\end{definition}
    
\begin{theorem}[\cite{VC,Haussler1986EpsilonnetsAS}]\label{thm:alphaApprox}
    Let $(X,R)$ have VC dimension $d$. Let $S\subseteq X$ be a subset of $X$. Let $0<\alpha,\beta\leq 1$. Let $S'\subseteq S$ be a random subset of $S$ with size at least 
    $ O\left(\frac{d\cdot\log\frac{d}{\alpha}+\log\frac{1}{\beta}}{\alpha^2}\right).
    $
    Then with probability at least $1-\beta$, $S'$ is an $\alpha$-approximation of $S$ for $R$.
 \end{theorem}

\begin{definition}[Generalization and empirical error]
    Let $\mathcal{D}$ be a distribution, $c$ be a concept and $h$ be a hypothesis. The error of $h$ w.r.t.~$c$ over $\mathcal{D}$ is defined as 
    $$
    \mathrm{err}_{\mathcal{D}}(c,h)=\Pr_{x\sim\mathcal{D}}[c(x)\neq h(x)].
    $$
    For a finite dataset $S$, the error of $h$ w.r.t.~$c$  over $S$ is defined as 
    $$
    \mathrm{err}_S(c,h)=\frac{\left|\{x\in S \mid c(x)\neq h(x)\}\right|}{|S|}.
    $$

    For convenience, we abbreviate $\mathrm{err}_{\mathcal{D}}(c,h)$ as $\mathrm{err}_{\mathcal{D}}(h)$ and abbreviate $\mathrm{err}_{S}(c,h)$ as $\mathrm{err}_{S}(h)$. We say a hypothesis $h$ is \textbf{$\alpha$-good} with respect to a distribution $\mathcal{D}$ (dataset $S$) if $\mathrm{err}_{\mathcal{D}}(c,h)\leq \alpha$ ($\mathrm{err}_{S}(c,h)\leq \alpha$, respectively).
\end{definition}

\begin{theorem}[\cite{BlumerEhHaWa89,kaplan2020private}]\label{thm:learn vc}
    Let $\mathcal{C}$ be a concept class and let $\mathcal{D}$ be a distribution. Let $\alpha,\beta>0$, and $m\geq \frac{48}{\alpha}\left(10VC(\mathcal{C})\log(\frac{48e}{\alpha})+\log(\frac{5}{\beta})\right)$. Let $S$ be a sample of $m$ points drawn i.i.d.\ from $\mathcal{D}$. 
    Then
    $$\Pr[\exists c,h\in\mathcal{C}~\mbox{s.t.}~\mathrm{err}_{S}(c,h)\leq\alpha/10 ~\mbox{and}~\mathrm{err}_{\mathcal{D}}(c,h)\geq\alpha]\leq \beta.$$
\end{theorem}

\subsection{Differential Privacy}

\begin{definition}[Differential Privacy~\cite{DMNS06}]
    Let $\calX$ be a data domain and $\calY$ be an output domain. A (randomized) mechanism $M$ mapping $\calX^n$ to $\calY$ is $(\varepsilon,\delta)$-differentially private if for any pair of inputs $S,S'\in\calX^n$  where $S$ and $S'$ differ on a single entry, and any event $E\subseteq \calY$, it holds that
    $$
    \Pr[M(S)\in E]\leq e^{\varepsilon}\cdot\Pr[M(S')\in E] +\delta,
    $$
    where the probability is over the randomness of $M$.
\end{definition}

\begin{theorem}[Advanced composition~\cite{DRV10}]\label{thm:advancedcomposition}
    Let $M_1,\dots,M_k:\calX^n\rightarrow \calY$ be $(\varepsilon,\delta)$-differentially private mechanisms. Then the algorithm that on input $S\in \calX^n$ outputs $(M_1(S),\dots,M_k(S))$ is $(\varepsilon',k\delta+\delta')$-differentially private, where $\varepsilon'=\sqrt{2k\ln(1/\delta')}\cdot \varepsilon + k\varepsilon\frac{e^\varepsilon-1}{e^\varepsilon+1}$ for every $\delta'>0$.

    Specifically, when $\varepsilon<\frac{1}{\sqrt{k}}$, we have $\varepsilon'=\theta(\sqrt{k\ln(1/\delta')}\cdot \varepsilon)$.
\end{theorem}

\subsubsection{Algorithm \texttt{BetweenThresholds}~\cite{BunSU16}}

\begin{algorithm}
\caption{\bf \texttt{BetweenThresholds}~\cite{BunSU16}}\label{alg:BetweenThresholds}
{\bf Input:} Database $S\in X^n$.\\
{\bf Parameters:} $\varepsilon,t_{\ell},t_{u} \in (0,1)$ and $n, T \in \mathbb{N}$.
\begin{enumerate}

\item Sample $\mu \sim Lap(2/\varepsilon n)$ and initialize noisy thresholds $\hat{t}_{\ell} = t_{\ell} + \mu$ and $\hat{t}_u = t_u - \mu$.

\item For $j = 1, 2, \cdots, T$:

    \begin{enumerate}

	\item Receive query $q_j : X^n \to [0,1]$.
	\item Set $c_j = q_j(S) + \nu_j$ where $\nu_j \sim Lap(6/\varepsilon n)$.
	\item If $c_j < \hat{t}_{\ell}$, output $L$ and continue.
	\item If $c_j > \hat{t}_{u}$, output $R$ and continue.
	\item If $c_j \in [\hat{t}_{\ell},\hat{t}_u]$, output $\top$ and halt.

    \end{enumerate}

\end{enumerate}
\end{algorithm}

\begin{lemma}[Privacy for \texttt{BetweenThresholds}~\cite{BunSU16}] \label{lem:bt-privacy}
Let $\varepsilon,\delta \in (0,1)$ and $n \in \mathbb{N}$. 
Then algorithm \texttt{BetweenThresholds} is $(\varepsilon, \delta)$-differentially private for any adaptively-chosen sequence of queries as long as the gap between the thresholds $t_{\ell}, t_u$ satisfies
\[t_u - t_{\ell} \geq \frac{12}{\varepsilon n}\left( \log (10/\varepsilon) + \log(1/\delta) + 1\right).\]
\end{lemma}

\begin{lemma}[Accuracy for \texttt{BetweenThresholds}~\cite{BunSU16}] \label{lem:bt-accuracy}
Let $\alpha, \beta,\varepsilon,t_{\ell},t_u \in (0,1)$ and $n,T \in \mathbb{N}$ satisfy \[n \geq \frac{8}{\alpha \varepsilon}\left(\log(T+1) + \log(1/\beta)\right).\] Then, for any input $x \in {X}^n$ and any adaptively-chosen sequence of queries $q_1, q_2, \cdots, q_T$, the answers $a_1, a_2, \cdots a_{\leq T}$ produced by \texttt{BetweenThresholds} on input $x$ satisfy the following with probability at least $1-\beta$. For any $j \in [T]$ such that $a_j$ is returned before \texttt{BetweenThresholds} halts,
\begin{itemize}
\item $a_j = L \implies q_j(x) \le t_{\ell} + \alpha$,
\item $a_j = R \implies q_j(x) \ge t_u - \alpha$, and
\item $a_j = \top \implies t_{\ell} - \alpha \le q_j(x) \le t_u + \alpha$.
\end{itemize}
\end{lemma}

\subsection{A Generic Framework}\label{sec:generic framework}
In this section, we present a generic prediction framework, building on the approaches of Dwork and Feldman~\cite{DworkF18} and Naor et al.~\cite{NNSY23}. The core idea is to randomly partition the labeled dataset into several disjoint subsets, learn a hypothesis non-privately from each subset, and then answer each query by aggregating the resulting hypotheses. Concretely, given a query point, we predict its label using a noisy majority vote over the ensemble.

As observed by Naor et al.~\cite{NNSY23}, the aggregation step can be implemented using the sparse vector technique, which allows us to answer many queries while spending privacy budget only on a limited number of 'hard' rounds. Consequently, both the privacy and accuracy guarantees of the framework are governed by the number of times the sparse vector subroutine triggers i.e., its stopping behavior.

\begin{algorithm}~\label{alg:generic predictor}
    \caption{GenericPredictor}

    \textbf{Inputs:} A labeled database $S\in (X\times\{\pm1\})^N$ where we set $k=\frac{64}{ \varepsilon}\left(\log(T+1) + \log(1/\beta)\right)$ and $N=k\cdot m$

    \begin{enumerate}
        \item Randomly partition $S$ to $S_1,\dots,S_k$, each of size $|S_i|=m$.

        \item For $j=1,\dots ,T$:

        \begin{enumerate}
            \item For $i=1,\dots ,k$, compute a hypothesis $G(S_i,transcript_j)=f_{i,j}\in\mathcal{C}$.
            Let $F_j=\{f_{1,j},\dots,f_{k,j}\}$
            
            /* $transcript_j$ includes all prediction queries and predicted labels before step $j$ 
            
            We view $F_j$ as a dataset with size $k$, change one example of $S$ will change one entry of $F_j$, so the query sensitivity is $1/k$
            */

            \item Receive a prediction query point $x_j\in X$.

            \item\label{step:generic betweenthresholds} Set parameter $t_\ell=1/2-1/8$ and $t_u=1/2+1/8$. Define the counting query $q_{x_j}(F_j)=\frac{1}{2}\left(1+\frac{1}{k}\sum_{i=1}^k f_{i,j}(x_j)\right)$, and run \texttt{BetweenThresholds} on $q_{x_j}$ to obtain an output $y_j\in\{L,\top,R\}$.

            \item If $y_j=L$, respond with the label $\mathsf{Label}_j=-1$.
            \item If $y_j=R$, respond with the label $\mathsf{Label}_j=1$.
            \item If $y_j=\top$, uniformly randomly respond with the label $\mathsf{Label}_j\in\{\pm1\}$.

            /* When \texttt{BetweenThresholds} outputs $\top$, we reinitialize it with fresh noise and continue.*/
        \end{enumerate}
    \end{enumerate}

\end{algorithm}

The following observation shows that both privacy and accuracy depend only on the number of times \texttt{BetweenThresholds} outputs $\top$. We defer the proof to the Appendix~\ref{appendix:generic predictor}.

\begin{observation}[\cite{NNSY23}]
    Let $v$ denote the number of times that \texttt{BetweenThresholds} in Step~\ref{step:generic betweenthresholds} outputs $\top$. Then Algorithm~\ref{alg:generic predictor} is equivalent to $v$ composition of \texttt{BetweenThresholds}.
\end{observation}

\begin{claim}[Privacy of Generic Predictor~\cite{NNSY23}]\label{clm:generic privacy}
    If \texttt{BetweenThresholds} in Step~\ref{step:generic betweenthresholds} outputs $\top$ at most $v$ times, then Algorithm~\ref{alg:generic predictor} is $\left(O(\sqrt{v\cdot\ln(1/\delta)}\cdot\varepsilon),O(v\delta)\right)$-differentially private.
\end{claim}

\begin{claim}[Accuracy of Generic Predictor~\cite{NNSY23}]\label{clm:generic accuracy}
    If all $f_{i,j}$ are $\alpha$-good with respect to the underlying distribution $\mathcal{D}$, then with probability $1-v\beta$, all predictors defined above are $4\alpha$-good with respect to the distribution $\mathcal{D}$.
\end{claim}

\section{Private Prediction with Oblivious Adversary}
In this section, we construct a hypothesis generator for an oblivious adversary. The key idea is that whenever \texttt{BetweenThresholds} outputs $\top$, we record this query and use it to shrink the current version space. In particular, each hard query induces a random constraint that (with constant probability) eliminates at least half of the remaining feasible labelings of the fixed query sequence.

Since the adversary is oblivious, the queries $x_1,\dots,x_T$ are fixed in advance, and the number of distinct label sequences realizable by $\mathcal{C}$ on these points is at most $O(T^{VC(\mathcal{C})})$ by Sauer's lemma. Therefore, after at most $O(VC(\mathcal{C})\log T)$ hard queries, only one labeling remains and all remaining hypotheses agree on future queries. Consequently, the number of times \texttt{BetweenThresholds} outputs $\top$ is bounded by $O(\mathrm{VC}(\mathcal{C})\log T)$ (see illustration in Figure~\ref{fig:oblivious 1} and \ref{fig:oblivious 2}).

\begin{figure}[ht]
\begin{center}
\includegraphics[scale=.5]{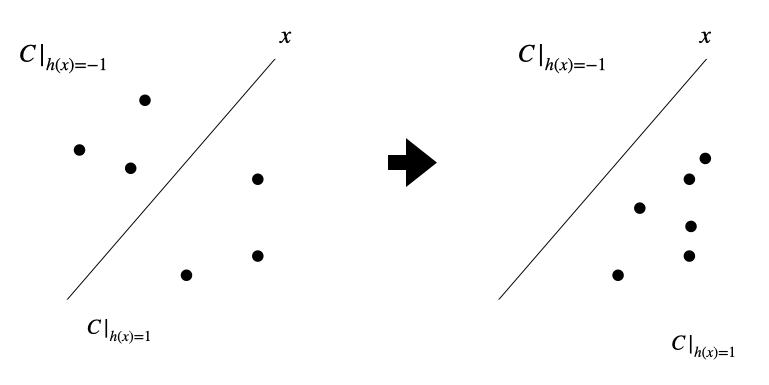}
\caption{In the left figure, we illustrate that when a “hard” query $x$ occurs, the current hypothesis set splits into two subspaces, $\mathcal{C}|_{h(x)=1}$ and $\mathcal{C}|_{h(x)=-1}$. We then guess a label uniformly at random (say, 1) and update the hypothesis set by restricting to $\mathcal{C}|_{h(x)=1}$, as shown in the right figure.\label{fig:oblivious 1}}
\end{center}
\end{figure}

\begin{figure}[ht]
\begin{center}
\includegraphics[scale=.5]{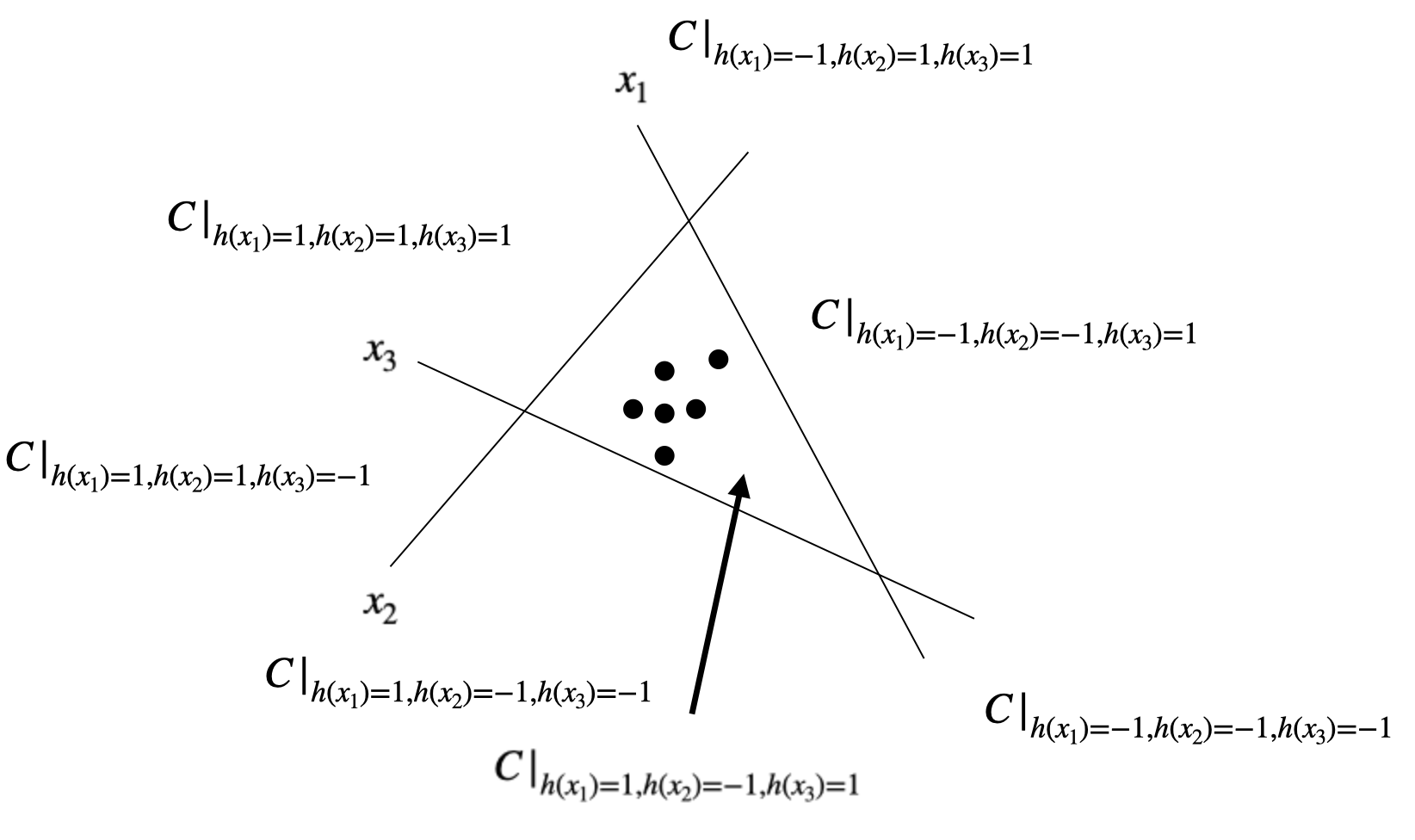}
\caption{After answering $O(VC(\mathcal{C})\log T)$ “hard” queries, the remaining hypothesis space collapses to a set of hypotheses that induce the same labeling on the entire query sequence $x_1,\dots, x_T$ (including queries that have not yet appeared in the stream).\label{fig:oblivious 2}}
\end{center}
\end{figure}

\subsection{Hypotheses Generator for Oblivious Adversary}
We use the generic framework of Section~\ref{sec:generic framework} and design a hypothesis generator for oblivious adversaries.
\begin{algorithm}
    \caption{$G_{oblivious}$: Hypotheses Generator for Oblivious Adversary}

    \textbf{Inputs:} A labeled database $S_i$, $transcript_j=\left((x_1,y_1,\mathsf{Label}_1),\dots,(x_{j-1},y_{j-1},\mathsf{Label}_{j-1})\right)$

    \begin{enumerate}
        \item Let $HardQuery=\{(x,y,\mathsf{Label})\in transcript_j| y=\top\}$. Let $\mathcal{C}|_{HardQuery}=\{c\in\mathcal{C}:c(x_{j^*})=\mathsf{Label}_{j^*}\mbox{ for all } (x_{j^*},y_{j^*},\mathsf{Label}_{j^*})\in HardQuery\}$

        \item Output a hypothesis $f_{i,j}=\arg\min_{h\in\mathcal{C}|_{HardQuery}}[ \mathrm{err}_{S_i}(h)]$.
    \end{enumerate}
\end{algorithm}

\subsection{Analysis}
Here we recall that we partition the dataset $S$ into $k$ subsets $S_1,\dots,S_k$, each with size $|S_i|=m$.
\begin{claim}~\label{clm:approximation of Si}
    If $m\geq O\left(\frac{d\cdot\log\frac{d}{\alpha}+\log\frac{1}{\beta}}{\alpha^2}\right)$, with probability $1-k\beta$, we have 
    \begin{enumerate}
        \item $\left| \mathrm{err}_{S}(h)- \mathrm{err}_{S_i}(h)\right|\leq\alpha$ for all $i\in[k]$ and all $h\in\mathcal{C}$.

        \item $\left| \mathrm{err}_{S_j}(h)- \mathrm{err}_{S_i}(h)\right|\leq2\alpha$ for all $i,j\in[k]$ and all $h\in\mathcal{C}$.
    \end{enumerate}

\end{claim}
\begin{proof}
    By Theorem~\ref{thm:alphaApprox}, with probability $1-\beta$, each $S_i$ is an $\alpha$-approximation of $S$ with respect to the concept class $\mathcal{C}$, and thus $\left|\mathrm{err}_{S}(h)-\mathrm{err}_{S_i}(h)\right|\leq\alpha$. The first statement can follow by union bound. The second statement is the corollary of the first statement.
\end{proof}

\begin{claim}
    Assume \texttt{BetweenThresholds} satisfies the accuracy guarantee of Lemma~\ref{lem:bt-accuracy}. Then $HardQuery$ is a set of realizable labeled points. That is, there exists a hypothesis $h\in\mathcal{C}$ satisfying $h(x_i)=\textsf{Label}_i$ for all pairs $(x_i,\top,\textsf{Label}_i)\in HardQuery$.
\end{claim}
\begin{proof}
    We prove it by induction on $|HardQuery|$. In the base case of $|HardQuery|=0$, $HardQuery$ is an empty set, and the result is trivial. 
    
    For the inductive step, suppose the claim holds for all transcripts with at most $a-1$ hard queries. Consider a transcript with $a$ hard queries, and let $(x_{a^*},\mathsf{Label}_{a^*})$ be the most recent one. At round $a^*$, \texttt{BetweenThresholds} output $\top$ in Step~\ref{step:generic betweenthresholds}, which means the ensemble vote $q_{x_{a^*}}=\frac{1}{2}(1+\frac{1}{k}\sum_{i=1}^k f_{i,a^*}(x_{a^*}))$ lies in the uncertain region, so in particular there exists at least one hypothesis $f_{i,a^*}$ with $f_{i,a^*}(x_{a^*})=\mathsf{Label}_{a^*}$. By construction, every $f_{i,a^*}$ is chosen from $\mathcal{C}|_{HardQuery\backslash{(x_{a^*},\top,\textsf{Label}_{a^*})}}$, which is nonempty by the induction hypothesis. Hence, there exists a concept in $\mathcal{C}$ consistent with all previous hard queries and also labeling $x_{a^*}$ as $\mathsf{Label}_{a^*}$, as required.
\end{proof}

\begin{claim}~\label{clm: stop in vclogt times}
    Assume \texttt{BetweenThresholds} satisfies Lemma~\ref{lem:bt-accuracy}. Then with probability $1-\beta$, \texttt{BetweenThresholds} outputs $\top$ at most $O(VC(\mathcal{C})\cdot\log T+\log(1/\beta))$ times.
\end{claim}
\begin{proof}
    Since all queries are selected in advance, we can consider all possible labels of $\mathcal{L}=\{h(x_1),\dots,h(x_T)\}_{h\in\mathcal{C}}$. By Sauer's lemma, we know $|\mathcal{L}|\leq O(T^{VC{(\mathcal{C})}})$. 

    Each time \texttt{BetweenThresholds} outputs $\top$, the algorithm outputs a uniformly random label. Conditioned on the history, with probability $1/2$ this random label matches the target label sequence, and the restriction $\mathcal{C}|_{\mathsf{HardQuery}}$ eliminates at least half of the remaining label sequences in $\mathcal{L}$. Let $E$ be the event that a given hard query halves $|\mathcal{L}|$, then $\Pr[E]\ge 1/2$.

    Therefore, after $O(\mathrm{VC}(\mathcal{C})\log T)$ successful halving events, the set of remaining label sequences has size $1$. On the other hand, when the number of hard queries is $O(VC(\mathcal{C})\cdot\log T+\log(1/\beta))$, the Chernoff bound implies that with probability at least $1-\beta$, the event $E$ happens at least $O(\mathrm{VC}(\mathcal{C})\log T)$ times.

    At the end, once $|\mathcal{L}|=1$, all remaining hypotheses agree on all future queries, all $f_{i,j}$ predict the same label for new $x_j$. under the guarantee of Lemma~\ref{lem:bt-accuracy}, \texttt{BetweenThresholds} will not output $\top$ anymore.
\end{proof}

Let $v=O(VC(\mathcal{C})\cdot\log T+\log(1/\beta))$, we have the privacy guarantee of the predictor with an oblivious adversary.

\begin{corollary}[Privacy Guarantee]
    Using $G_{oblivious}$, Algorithm~\ref{alg:generic predictor} is $(O(\sqrt{v\cdot\ln(1/\delta)}\cdot \varepsilon),O(v\cdot\delta))$-differentially private, where $v=O(VC(\mathcal{C})\cdot\log T+\log(1/\beta))$.
\end{corollary}

\begin{claim}~\label{clm: all fi are good}
    Under the guarantee of Claim~\ref{clm:approximation of Si} and Claim~\ref{clm: stop in vclogt times}, all $f_{i,j}$ satisfies $\mathrm{err}_{\mathcal{D}}(f_{i,j})\leq O(v\cdot\alpha)$, where $v=O(VC(\mathcal{C})\cdot\log T+\log(1/\beta))$..
\end{claim}
\begin{proof}
    By Theorem~\ref{thm:learn vc}, it suffices to prove $\mathrm{err}_S(f_{i,j})\leq O(v\cdot\alpha)$ for all $f_{i,j}$. Claim~\ref{clm: stop in vclogt times} implies that $|HardQuery|\leq v$.
    
    We prove it by induction. In the base case of $|HardQuery|=0$, all $f_{i,j}$ are selected from $\mathcal{C}$ and $\mathrm{err}_{S_i}(f_{i,j})=0$

    Assume when $|HardQuery|\leq c-1$, we have $\min_{h\in\mathcal{C}|_{HardQuery}}[\mathrm{err}_{S_i}(h)]\leq 2(c-1)\alpha$. For the case of  $|HardQuery|=c$, let $(x_{c^*},\top,\textsf{Label}_{c^*})$ be the $c$-th element of $HardQuery$. On the $c^*$-th query, \texttt{BetweenThresholds} output $\top$ in Step~\ref{step:generic betweenthresholds}. It happens when there exists $f_{i^*,c^*}(x_{c^*})=\textsf{Label}_{c^*}$ in query $q_{x_{c^*}}$ for some $i^*\in[k]$. Since $f_{i^*,c^*}$ is selected from $\mathcal{C}|_{HardQuery\backslash (x_{c^*},\top,\textsf{Label}_{c^*})}$, By the induction assumption, we have $\min_{h\in\mathcal{C}|_{HardQuery}}[\mathrm{err}_{S_{i^*}}(h)]\leq 2(c-1)\alpha$. By the selection strategy, we have $\mathrm{err}_{S_{i^*}}(f_{i^*,c^*})\leq 2(c-1)\alpha$. It also works for $f_{i^*,j}$, where $j\geq c^*$ and conditioned on $|HardQuery|=c$ because $HardQuery$ doesn't change. Thus by the second statement of Claim~\ref{clm:approximation of Si}, we have $\mathrm{err}_{S_{i^*}}(f_{i,x_j})\leq 2c\alpha$ for all $i\in[k]$. At the end, this claim can result from $c\leq v$ and the first statement of Claim~\ref{clm:approximation of Si}.

\end{proof}

 Using Claim~\ref{clm:generic accuracy}, we have the following corollary.
\begin{corollary} [Accuracy Guarantee]
    Using $G_{oblivious}$, with probability $1-O((v+k)\cdot\beta)$, every query $x_j$ is predicted by a predictor $h_j$ satisfying $\mathrm{err}_D(c,h_j) \le O(v\cdot\alpha)$
\end{corollary}

Select the appropriate parameters, and we can have the main result. We put the details in Appendix~\ref{appendix:oblivious}

\begin{theorem}
 For any concept class $\mathcal{C}$ with VC dimension $d$, there exists an $(\varepsilon,\delta)$-differentially private predictor that can $(\alpha,\beta)$-predict $T$ queries if the given dataset has size 
 $$
 \begin{array}{rl}
     N
     &=\tilde{O}_{\alpha,\beta,\varepsilon,\delta}\left(d^{3.5}\cdot \log^{3.5}T\right)
 \end{array}
 $$

\end{theorem}

\section{Prediction for Halfspaces with Adaptive Adversary}
In this section, we construct a hypothesis generator specialized to the halfspace class under an adaptive adversary. Our construction is based on the reduction from PAC learning halfspaces to linear feasibility problem and the corresponding definition $cdepth$, which are introduced by Kaplan et al.~\cite{kaplan2020private}. As in the generic framework, whenever \texttt{BetweenThresholds} outputs $\top$ we treat the corresponding query as a hard query and record it. The crucial difference is that, for halfspaces, each hard query can be interpreted as a geometric constraint in an associated linear-feasibility representation: the hard query identifies a hyperplane, and we restrict attention to the intersection of the approximate feasible region with this hyperplane.

By repeatedly intersecting with the hyperplanes induced by hard queries, we progressively collapse the approximate feasible region. In $\mathbb{R}^{d+1}$, after at most $d+1$ such independent intersections, the feasible set reduces to a single candidate, implying that all remaining hypotheses agree on every future query. Consequently, \texttt{BetweenThresholds} can output $\top$ only $d+1$ times.

\subsection{Additional Preliminaries}
For $x = (x_1,\ldots,x_d)$ and $y = (y_1,\ldots,y_d)$, we denote by $\langle{x,y}\rangle = \sum_{i=1}^d x_i y_i$ the inner-product of $x$ and $y$. For $a\in\mathbb{R}^d$ and $w\in\mathbb{R}$, we define the predicate $h_{a,w}(x) = \left(\langle a,x\rangle \geq w\right)$. Let $H_{a,w}$ be the halfspace $\{x\in\mathbb{R}^d \colon h_{a,w}(x) = 1\}$.

\begin{definition}[Halfspaces and hyperplanes]
    For $a_1,\dots,a_d,w\in\mathbb{R}$, the halfspace  predicate $h_{a_1,\dots,a_d,w}:\mathbb{R}^d\rightarrow\{\pm1\}$ is defined as $h_{a_1,\dots,a_d,w}(x_1,\dots,x_d)=1$ if and only if $\sum_{i=1}^d a_ix_i\geq w$. Define the concept class $\mbox{HALFSPACE}_d=\{h_{a_1,\dots,a_d,w}\}_{a_1,\dots,a_d,w\in\mathbb{R}}$. 
\end{definition}

\paragraph{Linear Feasibility Problem~\cite{kaplan2020private}.}
In a linear feasibility problem, we are given a feasible collection $S\in(\mathbb{R}^{d+1})^n$ of $n$ linear constraints over $d$ variables $x_1,\dots,x_d$. The target is to find a solution in $\mathbb{R}^d$ that satisfies all (or most) constraints. Each constraint has the form $h_{a,w}(x) = 1$ for some $a\in\mathbb{R}^d$ and $w\in \mathbb{R}$.

\begin{definition}[$(\alpha,\beta)$-solving $(d,n)$-linear feasibility~\cite{kaplan2020private}\footnote{In~\cite{kaplan2020private}, the authors consider the coefficients are in a domain $X$. In this work, we let $X$ to be $\mathbb{R}$.}]
    We say that an algorithm $(\alpha,\beta)$-solves $(d,n)$-linear feasibility if for every feasible collection of $n$ linear constraints over $d$ variables, with probability $1-\beta$ the algorithm finds a solution $(x_1,\dots,x_d)$ that satisfies at least $(1-\alpha)n$ constraints.
\end{definition}

Notice that there exists a reduction from PAC learning of halfspaces to solving linear feasibility. 
Each input labeled point $((x_1,\dots,x_d),y)\in\mathbb{R}^d\times\{-1,1\}$ is transformed to a linear constraint $h_{y\cdot(x_1,\dots,x_d,-1),0}$.
By Theorem~\ref{thm:learn vc}, if we can $(\alpha,\beta)$-solve $(d+1,n)$ linear feasibility problems, then the reduction results is an $(O(\alpha),O(\beta))$-PAC learner for $d$ dimensional halfspaces.

\begin{definition}[Convexification of a function $f:(\mathbb{R}^{d+1})^*\times\mathbb{R}^d\rightarrow \mathbb{R}$~\cite{kaplan2020private}]
    For $S\in(\mathbb{R}^{d+1})^*$ and $y\in\mathbb{R}$, let $\mathcal{D}_{S}(y)=\{z\in\mathbb{R}^d:f(S,z)\geq y\}$. The convexification of $f$ is the function $f_{Conv}:(\mathbb{R}^{d+1})^*\times\mathbb{R}^d\rightarrow \mathbb{R}$ defined by $f_{Conv}(S,x)=\max\{y\in\mathbb{R}:x\in ConvexHull(\mathcal{D}_{S(y)})\}$.
\end{definition}
I.e., if $x$ is in the convex hull of points $Z\subset \mathbb{R}^d$\remove{such that $f(S,z)\geq f(S,x)$ for all $z\in Z$} then $f_{Conv}(S,x)$ is at least $\min_{z\in Z}(f(S,z))$. 

\begin{definition}[$depth$ and $cdepth$~\cite{kaplan2020private}]\label{def:depth-cdepth}
Let $S$ be a collection of predicates. Define $depth_{S}(x)$ to be the number of predicates $h_{a,w}$ in $S$ such that $h_{a,w}(x)=1$.
Let $cdepth_{S}(x)=f_{Conv}(\mathcal{S},x)$ for the function $f(S,x)=depth_{\mathcal{S}}(x)$.
\end{definition}

\begin{fact}[\cite{kaplan2020private}]\label{fact:cdepth to depth}
    For any $S\in(\mathbb{R}^d\times \mathbb{R})^*$ and any $x\in\mathbb{R}^d$, it holds that 
    $$
    depth_{S}(x)\geq (d+1)\cdot cdepth_{S}(x)-d|S|.
    $$
\end{fact}

By the above fact, if we can find a point with $cdepth_{S}\geq (1-\alpha)|S|$ where $\alpha \ll 1/(d+1)$, then this point has $depth_S\approx |S|$.

\begin{lemma}[Approximation of cdepth~\cite{NTY25}]~\label{lem:vc for cdepth}
    Let $S\subseteq (\mathbb{R}^d)^n$ and 
    let $S'\subseteq S$ be a random subset of $S$ with cardinality $m=|S'|\geq O\left(\frac{d\cdot\log(\frac{d}{\alpha})+\log\frac{1}{\beta}}{\alpha^2}\right)$. 
    Then, with probability at least $1-\beta$, for all $p\in \mathbb{R}^d$, if $cdepth_{S'}(p)=\gamma' m$ and $cdepth_{S}(p)=\gamma n$ then $|\gamma-\gamma'|\leq\alpha$.
\end{lemma}

\subsection{Hypotheses Generator for Halfspaces Class}
We adapt the approach of Nissim et al.~\cite{NTY25} to design the hypothesis generator for halfspaces. 
The key step is to reduce learning a consistent halfspace to maintaining the feasibility of a suitable linear program (equivalently, a convex feasibility region in a dual parameter space). In this representation, a parameter point induces a hypothesis, and points of large $depth$ with respect to a labeled subset correspond to hypotheses with small empirical error on that subset. (See Figure~\ref{fig:halfspace} for illustration)

A technical complication is that updating according to $depth$ may incur an exponential increase of error, so we work with the convex version $cdepth$: at each round, we select a point of high $cdepth$ within the current feasible region. This choice has two crucial consequences. First, whenever \texttt{BetweenThresholds} outputs $\top$ (a hard query), the corresponding query induces a hyperplane constraint, and one can show that this hyperplane contains at least one point whose $cdepth$ remains large. Thus, we can safely shrink the approximate feasible region by intersecting it with this hyperplane without losing the existence of a high-$cdepth$ point. Second, high $cdepth$ implies high depth (by Fact~\ref{fact:cdepth to depth}), which in turn implies that the induced hypothesis has low empirical error on the block.

Finally, each hard query adds one hyperplane intersection to the feasible region. In $\mathbb{R}^{d+1}$, after at most $d+1$ such intersections (under the progress guarantee from~\cite{NTY25}), the feasible region collapses to essentially a single point, at which point all remaining hypotheses agree on future queries and \texttt{BetweenThresholds} no longer outputs $\top$.

\begin{figure}[ht]
\begin{center}
\includegraphics[scale=.2]{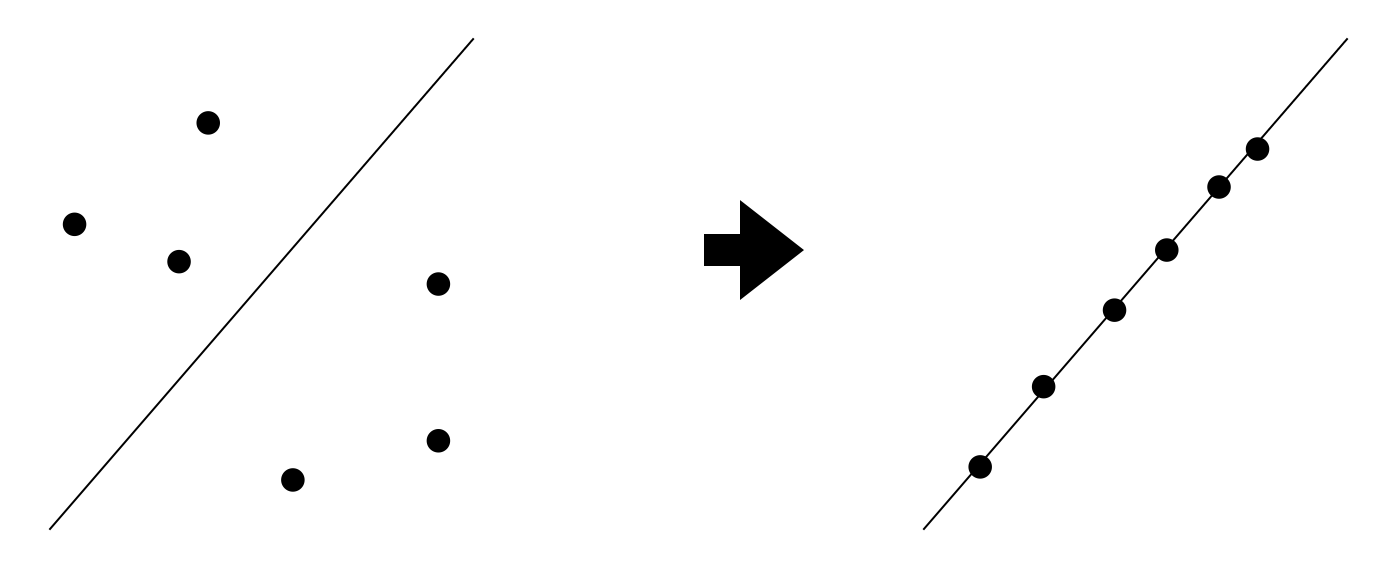}
\caption{In the left figure, we illustrate that when a “hard” query occurs, the current hypotheses are roughly split across the two sides of the induced hyperplane. We then update the feasible set by restricting all hypotheses to lie on this hyperplane for subsequent rounds, while preserving the existence of a hypothesis (point) with high $cdepth$, as shown in the right figure.\label{fig:halfspace}}
\end{center}
\end{figure}

\begin{algorithm}
    \caption{$G_{halfspace}$: Hypotheses Generator for Halfspaces Class}

    \textbf{Inputs:} A labeled database $S_i$, $transcript_j=\left((x_1,y_1,\textsf{Label}_1),\dots,(x_{j-1},y_{j-1},\textsf{Label}_{j-1})\right)$, where $x_1,\dots,x_{j-1}\in \mathbb{R}^d$, $y_1,\dots, y_{j-1}\in\{\pm 1,\top\},\;\textsf{Label}_1,\dots,\textsf{Label}_{j-1}\in\{\pm1\}$.

    \begin{enumerate}
        \item Let $HardQuery=\{(x,y,\textsf{Label})\in transcript_j| y=\top\}$. Let $\mathcal{C}|_{HardQuery}=\{c=(a_1,\dots,a_d,w):\langle a,x_{j^*}\rangle=w\mbox{ for all } (x_{j^*},y_{j^*},\textsf{Label}_{j^*})\in HardQuery\}$
        
        /* That is, the intersection space of all hyperplanes in $HardQuery$ and in the linear feasibility problem. */

        \item Output a hypothesis $f_{i,j}=\arg\max_{h\in\mathcal{C}|_{HardQuery}}[cdepth_{S_i}(h)]$. (That is, selecting
        
        $f_{i,j}$ from $\mathcal{C}|_{HardQuery}$, such that $cdepth_{S_i}(f_{i,j})=\max_{h\in\mathcal{C}|_{HardQuery}}[cdepth_{S_i}(h)])$
    \end{enumerate}
\end{algorithm}

\subsubsection{Analysis}
Here we recall that we partition the dataset $S$ into $k$ subsets $S_1,\dots,S_k$, each with size $|S_i|=m$.
\begin{claim}~\label{clm:approximation of cdepth}
    If $m\geq O\left(\frac{d\cdot\log\frac{d}{\alpha}+\log\frac{1}{\beta}}{\alpha^2}\right)$, then with probability $1-k\beta$, we have 
    \begin{enumerate}
        \item $\left|\frac{cdepth_{S}(h)}{N}-\frac{cdepth_{S_i}(h)}{m}\right|\leq\alpha$ for all $i\in[k]$ and all $h\in\mathcal{C}$.

        \item $\left|\frac{cdepth_{S_j}(h)}{m}-\frac{cdepth_{S_i}(h)}{m}\right|\leq2\alpha$ for all $i,j\in[k]$ and all $h\in\mathcal{C}$.
    \end{enumerate}

\end{claim}
\begin{proof}
    By Lemma~\ref{lem:vc for cdepth}, with probability $1-\beta$, we have $\left|\frac{cdepth_{S}(h)}{N}-\frac{cdepth_{S_i}(h)}{m}\right|\leq\alpha$ for any fixed $i\in[k]$. The first statement can follow by the union bound. The second statement is the corollary of the first statement.
\end{proof}

\begin{claim}
    When \texttt{BetweenThresholds} runs under the guarantee of Lemma~\ref{lem:bt-accuracy}, $\mathcal{C}|_{HardQuery}$ is always non-empty.
\end{claim}
\begin{proof}
    We prove it by induction on $|HardQuery|$. In the base case of $|HardQuery|=0$, then $\mathcal{C}|_{HardQuery}$ and the claim is immediate. 
    
    Assume the claim holds for the cases $|HardQuery|\leq a-1$, then we prove it also holds for $|HardQuery|= a$. Let $(x_{a^*},\top,\textsf{Label}_{a^*})$ be the $a$-th element of $HardQuery$. On the $a^*$-th query, \texttt{BetweenThresholds} output $\top$ in Step~\ref{step:generic betweenthresholds}. It happens when there exists $f_{i,a^*}(x_{{a^*}})=\textsf{Label}_{a^*}$ and $f_{i',a^*}(x_{{a^*}})=-\textsf{Label}_{a^*}$ in query $q_{x_{a^*}}$ for some $i,i'\in[k]$. 
    
    Viewing hypotheses as feasible points in the associated linear-feasibility (dual) space, the condition $f(x_{a^*})=\pm \mathsf{Label}_{a^*}$ corresponds to lying on opposite sides of the hyperplane induced by $x_{a^*}$. Hence the segment joining these two feasible points intersects that hyperplane. Let $p$ be an intersection point. Then $p$ satisfies all previous hard-query constraints and also satisfies the new constraint induced by $x_{a^*}$. Therefore $p\in \mathcal{C}|_{HardQuery}$, proving that the restricted class remains non-empty.

\end{proof}

\begin{claim}~\label{clm: halfspace stop in d times}
    When \texttt{BetweenThresholds} runs under the guarantee of Lemma~\ref{lem:bt-accuracy}, with probability $1-\beta$, \texttt{BetweenThresholds} outputs $\top$ at most $d+1$ times.
\end{claim}
\begin{proof}
    Notice that we maintain the feasible set as the intersection of the hyperplanes induced by the hard queries. Consider a new query $x$. If the hyperplane induced by $x$ is redundant with the existing constraints (e.g., it is parallel to, or coincides with, one of the previously imposed hyperplanes in the relevant feasibility representation), then the feasible region lies entirely on one side of the decision boundary for $x$. In this case, all feasible hypotheses agree on the label of $x$, so the ensemble vote is far from the threshold and \texttt{BetweenThresholds} will not output $\top$.

    More generally, \texttt{BetweenThresholds} can output $\top$ only when the new hard query induces a \emph{non-redundant} hyperplane constraint, i.e., one that strictly shrinks the feasible region. Each such non-redundant constraint decreases the affine dimension of the feasible region by at least one. Since the feasible region lives in $\mathbb{R}^{d+1}$, this can happen at most $d+1$ times. Once the affine dimension drops to $0$, the feasible region collapses to (at most) a single point, and hence all remaining hypotheses agree on every future query; therefore \texttt{BetweenThresholds} will never output $\top$ again.

\end{proof}

Let $v=d+1$, we have the privacy guarantee of the predictor for halfspace class with an adaptive adversary .

\begin{corollary}[Privacy Guarantee]
    Using $G_{halfspace}$, Algorithm~\ref{alg:generic predictor} is $(O(\sqrt{d\cdot\ln(1/\delta)}\cdot \varepsilon),O(d\cdot\delta))$-differentially private
\end{corollary}

\begin{claim}~\label{clm: cdepth of fi}
    Under the guarantee of Claim~\ref{clm:approximation of cdepth} and Claim~\ref{clm: halfspace stop in d times}, all $f_{i,j}$ satisfies $\frac{cdepth_S(f_{i,j})}{N}\geq 1-(2d+3)\alpha$
\end{claim}
\begin{proof}  
    We prove the stronger inductive statement: if $|HardQuery|=c$, then for every $i\in[k]$, we have $\frac{\max_{h\in \mathcal{C}|_{HardQuery}}[cdepth_{S_i}(h)]}{m} \geq 1-2c\alpha$. Since $G_{halfspace}$ outputs a maximizer of $cdepth_{S_i}$ over the restricted set, this implies $\frac{cdepth_{S_i}(f_{i,j})}{m}\ge 1-2c\alpha$, and then Claim~\ref{clm:approximation of cdepth} transfers the bound from $S_i$ to $S$.

    In the base case of $c=0$, we optimize over $\mathcal{C}$, and there exists a feasible hypothesis achieving $cdepth_{S_i}(h)=m$. Hence the claim holds.

    Assume the statement holds for $c-1$, and consider the first time $|HardQuery|=c$. Let $(x_{t^\star},\top,\mathsf{Label}_{t^\star})$ be the $c$-th hard query added. At round $t^\star$, \texttt{BetweenThresholds} outputs $\top$. Under Lemma~\ref{lem:bt-accuracy}, this implies that the ensemble is uncertain on $x_{t^\star}$, and in particular there exist two generated hypotheses $f_{i^\star,t^\star}$ and $f_{i^{\star\star},t^\star}$ such that $f_{i^\star,t^\star}(x_{t^\star})=\mathsf{Label}_{t^\star}$ and $f_{i^{\star\star},t^\star}(x_{t^\star})=-\mathsf{Label}_{t^\star}$.

    In the associated feasibility space, these correspond to two feasible points lying on opposite sides of the hyperplane $H_{x_{t^\star}}$ induced by the query. Let $p_{\mathsf{good}}$ be an intersection point of $H_{x_{t^\star}}$ with the segment joining these two feasible points. Then $p_{\mathsf{good}}$ satisfies all constraints in $\mathsf{HardQuery}\setminus\{(x_{t^\star},\top,\mathsf{Label}_{t^\star})\}$ and also lies on $H_{x_{t^\star}}$, hence $p_{\mathsf{good}}\in \mathcal{C}|_{HardQuery}$.

    By the convexity property of $cdepth$ in the feasibility representation, we have 
    $$
    \begin{array}{rl}
         \frac{cdepth_{S}(p_{good})}{N}
         &\geq \min\{\frac{cdepth_{S}(f_{i^\star,t^\star})}{N},\frac{cdepth_{S}(f_{i^{\star\star},t^\star})}{N}\}  \\
         &\geq \min\{\frac{cdepth_{S_{i^\star}}(f_{i^\star,t^\star})}{m}-\alpha,\frac{cdepth_{S_{i^{\star\star}}}(f_{i^{\star\star},t^\star})}{m}-\alpha \}\\
         &\geq 1-(2c-1)\alpha  
    \end{array}
    $$
    It implies $\frac{\max_{h\in\mathcal{C}|_{HardQuery}}[cdepth_{S_i}(h)]}{m}\geq \frac{cdepth_{S}(p_{good})}{N}-\alpha\geq1-2c\alpha$ for all $i\in[k]$.
     At the end, this claim can result from $c\leq d+1$.

\end{proof}

Using Fact~\ref{fact:cdepth to depth}, we have the following corollary:
\begin{corollary}
    Under the guarantee of Claim~\ref{clm:approximation of cdepth} and Claim~\ref{clm: halfspace stop in d times}, all $f_{i,j}$ satisfies $depth_S(f_{i,j})\geq (1-O(d^2\alpha))N$
\end{corollary}

By Theorem~\ref{thm:learn vc}, it implies

\begin{claim}~\label{clm: halfspace all fi are good}
    Under the guarantee of Claim~\ref{clm:approximation of cdepth} and Claim~\ref{clm: halfspace stop in d times}, all $f_{i,j}$ are  $O(d^2\cdot\alpha)$-good with respect to the distribution $\mathcal{D}$.
\end{claim}

 Using Claim~\ref{clm:generic accuracy}, we have the accuracy guarantee:
\begin{corollary} [Accuracy Guarantee]
    Using $G_{halfspace}$, with probability $1-O((d+k)\cdot\beta)$, all queries are predicted by a predictor that is $O(d^2\cdot\alpha)$-good with respect to the distribution $\mathcal{D}$.
\end{corollary}

Select the appropriate parameters, and we can have the main result. We put the details in Appendix~\ref{appendix:halfspace}

\begin{theorem}
 For halfspaces class $\mathcal{C}$ over domain $\mathbb{R}^d$, there exists an $(\varepsilon,\delta)$-differentially private algorithm that can $(\alpha,\beta)$ predict $T$ queries if the give dataset has size 
 $$
 \begin{array}{rl}
     N&=\tilde{O}_{\alpha,\beta,\varepsilon,\delta}\left(d^{5.5}\cdot \log T\right)
 \end{array}
 $$

\end{theorem}

\bibliographystyle{plain}
\bibliography{refs}

\appendix
\section{Appendix}
\subsection{Claims on Generic Predictor}\label{appendix:generic predictor}
\begin{claim}[Privacy of Generic Predictor~\cite{NNSY23}, restate of Claim~\ref{clm:generic privacy}]
    If \texttt{BetweenThresholds} in Step~\ref{step:generic betweenthresholds} outputs $\top$ at most $v$ times, then Algorithm~\ref{alg:generic predictor} is $\left(O(\sqrt{v\cdot\ln(1/\delta)}\cdot\varepsilon),O(v\delta)\right)$-differentially private.
\end{claim}
\begin{proof}
    Notice that one different point makes one $f_{i,j}$ different. This claim can result from advanced composition (Theorem~\ref{thm:advancedcomposition}).
\end{proof}

\begin{claim}[Accuracy of Generic Predictor~\cite{NNSY23}, restate of Claim~\ref{clm:generic accuracy}]
    If all $f_{i,j}$ are $\alpha$-good with respect to the underlying distribution $\mathcal{D}$, then with probability $1-v\beta$, all predictors defined above are $4\alpha$-good with respect to the distribution $\mathcal{D}$.
\end{claim}
\begin{proof}
Let $\mathcal{G}$ be the event that every invocation of \texttt{BetweenThresholds} satisfies the accuracy guarantee
of Lemma~\ref{lem:bt-accuracy}. By a union bound over at most $v$ invocations, $\Pr[\mathcal{G}]\ge 1-v\beta$.
Condition on $\mathcal{G}$.

Fix a round $j$. Note that $ q_x(F_j)$ is exactly the fraction of hypotheses in $F_j$ that output $+1$ on $x$.

Run \texttt{BetweenThresholds} on $q_x(F_j)$ with thresholds $t_\ell=3/8$ and $t_u=5/8$
(and let $\tau=1/8$ denote the accuracy parameter appearing in Lemma~\ref{lem:bt-accuracy}).
Under $\mathcal{G}$, Lemma~\ref{lem:bt-accuracy} implies:
\begin{itemize}
    \item if the output is $L$, then $q_x(F_j)\le t_\ell+\tau$;
    \item if the output is $R$, then $q_x(F_j)\ge t_u-\tau$;
    \item if the output is $\top$, then $q_x(F_j)\in[t_\ell-\tau,\,t_u+\tau]$.
\end{itemize}

We claim that whenever the algorithm outputs an \emph{incorrect} label on $x$, at least a $1/4$-fraction of the
hypotheses $\{f_{i,j}\}_{i=1}^k$ are incorrect on $x$ (with respect to $c$), provided $\tau\le 1/8$.
Indeed:
\begin{itemize}
    \item If the algorithm outputs $-1$ (output $L$) but $c(x)=+1$, then the fraction of $-1$ votes is
    $1- q_x(F_j)\ge 1-(t_\ell+\tau)=5/8-\tau\ge 1/4$.
    \item If the algorithm outputs $+1$ (output $R$) but $c(x)=-1$, then the fraction of $+1$ votes is
    $ q_x(F_j)\ge t_u-\tau=5/8-\tau\ge 1/4$.
    \item If the output is $\top$, the algorithm outputs a uniformly random label; if this label is wrong,
    then either $c(x)=+1$ and the wrong label is $-1$, or vice versa. In either case, since
    $ q_x(F_j)\in[t_\ell-\tau,\,t_u+\tau]$, the fraction of votes for the wrong label is at least
    $\min\{ q_x(F_j),1- q_x(F_j)\}\ge 3/8-\tau\ge 1/4$.
\end{itemize}

Therefore, for every $x$,
\[
\mathbf{1}[\text{algorithm errs on }x]
\;\le\;
4\cdot \frac{1}{k}\sum_{i=1}^k \mathbf{1}[f_{i,j}(x)\neq c(x)].
\]
Taking expectation over $x\sim\mathcal{D}$ and using $\mathrm{err}_{\mathcal{D}}(c,f_{i,j})\le\alpha$ for all $i$ yields
\[
\mathrm{err}_{\mathcal{D}}(c,h_j)
\;\le\;
4\cdot \frac{1}{k}\sum_{i=1}^k \mathrm{err}_{\mathcal{D}}(c,f_{i,j})
\;\le\; 4\alpha.
\]
Since this holds for every $j\in[T]$ under $\mathcal{G}$, the claim follows.
\end{proof}

\subsection{Parameter Setting for the Predictor with Oblivious Adversary}\label{appendix:oblivious}
Let $VC(\mathcal{C})=d$. Set parameter
\begin{itemize}
    \item $\beta_{BT}=\frac{\beta\varepsilon}{(d\log T+\log(d\log T/\alpha\beta\varepsilon\delta))\sqrt{(d\log T+\log(d\log T/\alpha\beta\varepsilon\delta))\cdot\log(d\log T/\alpha\beta\varepsilon\delta)}\cdot\log(d\log T/\alpha\beta\varepsilon\delta)}$
    \item $\varepsilon_{BT}=\frac{\varepsilon}{\sqrt{(d\cdot\log T+\log(d\log T/\alpha\beta\varepsilon\delta))\cdot\ln(d\log T/\alpha\beta\varepsilon\delta)}}$
    \item $\delta_{BT}=\frac{\delta}{d\cdot\log T+\log(d\log T/\alpha\beta\varepsilon\delta)}$
    \item $\alpha_{BT}=\frac{\alpha}{d\cdot\log T+\log(d\log T/\alpha\beta\varepsilon\delta)}$
\end{itemize}

Then we have the number of subsets 
$$
\begin{array}{rl}
    k&= \frac{64}{ \varepsilon_{BT}}\left(\log(T+1) + \log(1/\beta_{BT})\right) \\
     & = O(\frac{\sqrt{(d\cdot\log T+\log(d\log T/\alpha\beta\varepsilon\delta))\cdot\ln(d\log T/\alpha\beta\varepsilon\delta)}\cdot(\log T+\log(\frac{d\log\log T\log\log(1/\alpha\delta)}{\beta\varepsilon})}{\varepsilon})\\
     &=\tilde{O}_{\alpha,\beta,\varepsilon,\delta}(d^{0.5}\cdot \log^{1.5} T)
\end{array}
$$
and the size of each subset 
$$
\begin{array}{rl}
    m&= O\left(\frac{d\cdot\log\frac{d}{\alpha_{BT}}+\log\frac{1}{\beta_{BT}}}{\alpha_{BT}^2}\right) \\
     & = O\left(\frac{(d\cdot\log T+\log(d\log T/\alpha\beta\varepsilon\delta))^2\cdot(d\log\frac{d}{\alpha\beta\varepsilon\delta}+\log\frac{d\log\log T\log\log(1/\alpha\delta)}{\beta\varepsilon})}{\alpha^2}\right)\\
     & =\tilde{O}_{\alpha,\beta,\varepsilon,\delta}(d^3\cdot \log^2 T)
\end{array}
$$
Then we have
$N=mk=\tilde{O}_{\alpha,\beta,\varepsilon,\delta}\left(d^{3.5}\cdot \log^{3.5}T\right)$

\subsection{Parameter Setting for the Predictor of halfspace class}\label{appendix:halfspace}
Set parameter
\begin{itemize}
    \item $\beta_{BT}=\frac{\beta\varepsilon}{d\log T\sqrt{d\log(d\log T/\delta)}\cdot(\log d+\log\log T+\log\log(1/\delta)+\log(1/\varepsilon))}$
    \item $\varepsilon_{BT}=\frac{\varepsilon}{\sqrt{d\cdot\ln(d/\delta)}}$
    \item $\delta_{BT}=\frac{\delta}{d}$
    \item $\alpha_{BT}=\frac{\alpha}{d^2}$
\end{itemize}

Then we have the number of subsets 
$$
\begin{array}{rl}
    k&= \frac{64}{ \varepsilon_{BT}}\left(\log(T+1) + \log(1/\beta_{BT})\right) \\
     & = O(\frac{\sqrt{d\ln(d/\delta)}\cdot(\log T+\log(\frac{d\log\log T\log\log(1/\delta)}{\beta\delta})}{\varepsilon})\\
     &=\tilde{O}_{\alpha,\beta,\varepsilon,\delta}(d^{0.5}\cdot \log T)
\end{array}
$$
and the size of each subset (we ignore the $O(\log d)$ and $O(\log\log T)$ part)
$$
\begin{array}{rl}
    m&= O\left(\frac{d\cdot\log\frac{d}{\alpha_{BT}}+\log\frac{1}{\beta_{BT}}}{\alpha_{BT}^2}\right) \\
     & = O\left(\frac{d^4\cdot(d\log\frac{d}{\alpha}+\log(\frac{d\log\log T\log\log(1/\delta)}{\beta\delta}))}{\alpha^2}\right)\\
     & =\tilde{O}_{\alpha,\beta,\varepsilon,\delta}(d^5)
\end{array}
$$ 
Then we have
$N=mk=\tilde{O}_{\alpha,\beta,\varepsilon,\delta}\left(d^{5.5}\cdot \log T\right)$

\end{document}